\documentclass[twoside]{article}

\usepackage[utf8]{inputenc}
\usepackage[accepted]{aistats2019}

\usepackage{amssymb,amsmath,mathtools,amsthm}
\usepackage{xcolor}
\usepackage{graphicx}
\usepackage[font=small]{caption}
\usepackage{algpseudocode,algorithm,algorithmicx}
\usepackage{multirow}
\usepackage{hyperref}

\newcommand{\KL}{{\mathrm{KL}}}
\newcommand{\mean}{\mathbb{E}}
\newcommand{\var}{{\rm I\kern-.3em D}}
\newcommand{\RR}{\mathbb{R}}
\newcommand{\cond}{\,|\,}
\newcommand{\Normal}{\mathcal{N}}
\newcommand{\diag}{\mathrm{diag}}

\newcommand{\eps}{\varepsilon}

\newtheorem{theorem}{Theorem}



\begin{document}

%

%

\twocolumn[

\aistatstitle{Doubly Semi-Implicit Variational Inference}


\aistatsauthor{ Dmitry Molchanov$^{1,2,*}$ \And Valery Kharitonov$^{2,*}$ \And Artem Sobolev$^{1}$ \And Dmitry Vetrov$^{1,2}$ }

\aistatsaddress{\hspace{-2em}dmolch111@gmail.com \ \ \ \ \ \ \ \ \ \ kharvd@gmail.com \ \ \ \ \ \ \ \ \ \ asobolev@bayesgroup.ru \ \ \ \ \ \ vetrovd@yandex.ru\\
$^{1}$ Samsung AI Center Moscow \\
$^{2}$ Samsung-HSE Laboratory, National Research University Higher School of Economics\\
$^*$ Equal contribution}

\runningauthor{ Dmitry Molchanov, Valery Kharitonov, Artem Sobolev, Dmitry Vetrov }


]

\begin{abstract}
We extend the existing framework of semi-implicit variational inference (SIVI) and introduce doubly semi-implicit variational inference (DSIVI), a way to perform variational inference and learning when both the approximate posterior and the prior distribution are semi-implicit.
In other words, DSIVI performs inference in models where the prior and the posterior can be expressed as an intractable infinite mixture of some analytic density with a highly flexible implicit mixing distribution.
We provide a sandwich bound on the evidence lower bound (ELBO) objective that can be made arbitrarily tight.
Unlike discriminator-based and kernel-based approaches to implicit variational inference, DSIVI optimizes a proper lower bound on ELBO that is asymptotically exact.
We evaluate DSIVI on a set of problems that benefit from implicit priors.
In particular, we show that DSIVI gives rise to a simple modification of VampPrior, the current state-of-the-art prior for variational autoencoders, which improves its performance.

\end{abstract}

\section{INTRODUCTION}
Bayesian inference is an important tool in machine learning.
It provides a principled way to reason about uncertainty in parameters or hidden representations.
In recent years, there has been great progress in scalable Bayesian methods, which made it possible to perform approximate inference for large-scale datasets and deep learning models. 

One of such methods is variational inference (VI)~\cite{blei2017}, which is an optimization-based approach. Given a probabilistic model $p(x, z)=p(x\cond z)p(z)$, where $x$ are observed data and $z$ are latent variables, VI seeks to maximize the evidence lower bound (ELBO).
\begin{equation}
    \label{eq:elbo_intro}
    \mathcal{L}(\phi)=\mean_{q_\phi(z)}[\log p(x\cond z)] - \KL(q_{\phi}(z) \,\|\, p(z)),
\end{equation}
where $q_\phi(z)$ approximates the intractable true posterior distribution $p(z\cond x)$.
The parametric approximation family for $q_\phi$ is chosen in such a way, that we can efficiently estimate $\mathcal{L}(\phi)$ and its gradients w.r.t.~$\phi$.

Such approximations to the true posterior are often too simplistic.
There exists a variety of ways to extend the variational family to mitigate this.
They can be divided roughly into two main groups: those that require the probability density function of the approximate posterior to be analytically tractable (which we will call \emph{explicit} models) \cite{saul1996,jaakkola1998,hoffman2015,giordano2015,tran2015,tran2016,han2016,ranganath2016,maaloe2016} and those that do not (\emph{implicit} models) \cite{huszar2017,mohamed2016,tran2017,li2017,mescheder2017,shi2017,yin2018}.
For latter, we only assume that it is possible to sample from such distributions, whereas the density may be inaccessible.

Not only approximate posteriors but also priors in such models are often chosen to be very simple to make computations tractable.
This can lead to overregularization and poor hidden representations in generative models such as variational autoencoders (VAE, \cite{kingma2013}) \cite{hoffman2016,tomczak2017,alemi2018}.
In Bayesian neural networks, a standard normal prior is the default choice, but together with the mean field posterior, it can lead to overpruning and consequently underfitting \cite{trippe2018}.
To overcome such problem in practice, one usually scales the KL divergence term in the expression for ELBO or truncates the variances of the approximate posterior \cite{neklyudov2017,louizos2017multiplicative,louizos2017compression}.

Another way to overcome this problem is to consider more complicated prior distributions, e.g. implicit priors.
For example, hierarchical priors usually impose an implicit marginal prior when hyperparameters are integrated out.
To perform inference in such models, one often resorts to joint inference over both parameters and hyperparameters, even though we are only interested in the marginal posterior over parameters of the model.
Another example of implicit prior distributions is the optimal prior for variational autoencoders.
It can be shown that the aggregated posterior distribution is the optimal prior for VAE \cite{hoffman2016}, and it can be regarded as an implicit distribution.
The VampPrior model \cite{tomczak2017} approximates this implicit prior using an explicit discrete mixture of Gaussian posteriors.
However, this model can be further improved if we consider an arbitrary trainable semi-implicit prior.

In this paper, we extend the recently proposed framework of semi-implicit variational inference (SIVI)~\cite{yin2018} and consider priors and posteriors that are defined as semi-implicit distributions.
By~``semi-implicit'' we mean distributions that do not have a tractable PDF (i.e. implicit), but that can be represented as a mixture of some analytically tractable density with a flexible mixing distribution, either explicit or implicit.

Our contributions can be summarized as follows.
Firstly, we prove that the SIVI objective is actually a lower bound on the true ELBO, which allows us to sandwich the ELBO between an upper bound and a lower bound which are both asymptotically exact.
Secondly, we propose \emph{doubly semi-implicit variational inference} (DSIVI), a general-purpose framework for variational inference and variational learning in the case when both the posterior and the prior are semi-implicit.
We construct a SIVI-inspired asymptotically exact lower bound on the ELBO for this case, and use the variational representation of the KL divergence to obtain the upper bound.
Finally, we consider a wide range of applications where semi-implicit distributions naturally arise, and show how the use of DSIVI in these settings is beneficial.

\section{PRELIMINARIES}
Consider a probabilistic model defined by its joint distribution $p(x, z)=p(x\cond z)p(z)$, where variables $x$ are observed, and $z$ are the latent variables.
Variational inference is a family of methods that approximate the intractable posterior distribution $p(z\cond x)$ with a tractable parametric distribution $q_\phi(z)$.
To do so, VI methods maximize the evidence lower bound (ELBO):
\begin{equation}
    \label{eq:elbo}
    \log p(x)\geq \mathcal{L}(\phi)=\mean_{q_\phi(z)}\log\frac{p(x\cond z)p(z)}{q_\phi(z)}\to\max_{\phi}.
\end{equation}
The maximum of the  evidence lower bound corresponds to the minimum of the KL-divergence $\KL(q_\phi(z)\,\|\,p(z\cond x))$ between the variational distribution $q_\phi(z)$ and the exact posterior $p(z\cond x)$.
In the more general \emph{variational learning} setting, the prior distribution may also be a parametric distribution $p_\theta(z)$ \cite{huszar2017}.
In this case, one would optimize the ELBO w.r.t. both the variational parameters $\phi$ and the prior parameters $\theta$, thus performing approximate maximization of the marginal likelihood $p(x\cond\theta)$.

The common way to estimate the gradient of this objective is to use the reparameterization trick \cite{kingma2013}.
The reparameterization trick recasts the sampling from the parametric distribution $q_\phi(z)$ as the sampling of non-parametric noise $\eps\sim p(\eps)$, followed by a deterministic parametric transformation $z=f(\eps,\phi)$.
Still, such gradient estimator requires log-densities of both the prior distribution $p(z)$ and the approximate posterior $q_\phi(z)$ in closed form.
Several methods have been proposed to overcome this limitation \cite{ranganath2016,mescheder2017,shi2017}.
However, such methods usually provide a biased estimate of the evidence lower bound with no practical way of estimating the introduced bias.

The reparameterizable distributions with no closed-form densities are usually referred to as \emph{implicit distributions}.
In this paper we consider the so-called semi-implicit distributions that are defined as an implicit mixture of explicit conditional distributions:
\begin{equation}
\label{eq:si}
    q_\phi(z)=\int q_\phi(z\cond\psi)q_\phi(\psi)\,d\psi.
\end{equation}
Here, the conditional distribution $q_\phi(z\cond\psi)$ is explicit.
However, when its condition $\psi$ follows an implicit distribution $q_\phi(\psi)$, the resulting marginal distribution $q_\phi(z)$ is implicit.
We will refer to $q_\phi(\psi)$ as the mixing distribution, and to $\psi$ as the mixing variables.

Note that we may easily sample from semi-implicit distributions: in order to sample $z$ from $q_\phi(z)$, we need to first sample the mixing variable $\psi\sim q_\phi(\psi)$, and then sample $z$ from the conditional $q_\phi(z\cond\psi)$.
Further in the text, we will assume this sampling scheme when using expectations $\mean_{z\sim q_\phi(z)}$ over semi-implicit distributions.
Also note that an arbitrary implicit distribution can be represented in a semi-implicit form: $q_\phi(z)=\int \delta(z-z')q_\phi(z')\,dz'$.

\section{RELATED WORK}

There are several approaches to inference and learning in models with implicit distributions.

One approach is commonly referred to as hierarchical variational inference or auxiliary variable models.
It allows for inference with implicit approximate posteriors $q_\phi(z)$ that can be represented as a marginal distribution of an explicit joint distribution $q_\phi(z)=\int q_\phi(z,\psi)\,d\psi$.
The ELBO is then bounded from below using a reverse variational model $r_\omega(\psi\cond z)\approx q_\phi(\psi\cond z)$
\cite{ranganath2016,salimans2015,louizos2017multiplicative}.
This method does not allow for implicit prior distributions, requires access to the explicit joint density $q_\phi(z,\psi)$ and has no way to estimate the increased inference gap, introduced by the imperfect reverse model.
However, recently proposed deep weight prior \cite{atanov2018the} provides a new lower bound, suitable for learning hierarchical priors in a similar fashion.

Another family of models uses an optimal discriminator to estimate the ratio of implicit densities $r(z)=\frac{q_\phi(z)}{p_\theta(z)}$ \cite{mescheder2017,mohamed2016,huszar2017}.
This is the most general approach to inference and learning with implicit distributions, but it also optimizes a biased surrogate ELBO, and the induced bias cannot be estimated.
Also, different authors report that the performance of this approach is poor if the dimensionality of the implicit densities is high \cite{sugiyama2012,tran2017}.
This is the only approach that allows to perform variational learning (learning the parameters $\theta$ of the prior distribution $p_\theta(z)$).
However, it is non-trivial and requires differentiation through a series of SGD updates.
This approach has not been validated in practice yet and has only been proposed as a theoretical concept \cite{huszar2017}.
On the contrary, DSIVI provides a lower bound that can be directly optimized w.r.t. both the variational parameters $\phi$ and the prior parameters $\theta$, naturally enabling variational learning.

Kernel implicit variational inference (KIVI) \cite{shi2017} is another approach that uses kernelized ridge regression to approximate the density ratio.
It is reported to be more stable than the discriminator-based approaches, as the proposed density ratio estimator can be computed in closed form.
Still, this procedure introduces a bias that is not addressed.
Also, KIVI relies on adaptive contrast that does not allow for implicit prior distributions \cite{mescheder2017,shi2017}.

There are also alternative formulations of variational inference that are based on different divergences.
One example is operator variational inference \cite{ranganath2016operator} that uses the Langevin-Stein operator to design a new variational objective.
Although it allows for arbitrary implicit posterior approximations, the prior distribution has to be explicit.

\section{DOUBLY SEMI-IMPLICIT VARIATIONAL INFERENCE}
In this section, we will describe semi-implicit variational inference, study its properties, and then extend it for the case of semi-implicit prior distributions.

\subsection{Semi-Implicit Variational Inference}
Semi-implicit variational inference \cite{yin2018} considers models with an explicit joint distribution $p(x, z)$ and a semi-implicit approximate posterior $q_\phi(z)$, as defined in Eq.~\eqref{eq:si}.
The basic idea of semi-implicit variational inference is to approximate the semi-implicit approximate posterior with a finite mixture:
\begin{equation}
\begin{aligned}
    q_\phi(z)&=\int q_\phi(z\,|\,\psi)q_\phi(\psi)\,d\psi\approx\\
    &\approx\frac1K\sum_{k=1}^Kq_\phi(z\,|\,\psi^k),\ \ \  \psi^k\sim q_\phi(\psi).
\end{aligned}
\end{equation}
SIVI provides an upper bound $\overline{\mathcal{L}}^q_K\geq\overline{\mathcal{L}}^q_{K+1}\geq\mathcal{L}$, and a surrogate objective $\underline{\mathcal{L}}^q_K$ that both converge to ELBO as $K$ goes to infinity ($\underline{\mathcal{L}}^q_\infty=\overline{\mathcal{L}}^q_\infty=\mathcal{L}$):
\begin{align}
    &\overline{\mathcal{L}}^q_K=\mean_{q_\phi(z)}\log{p(x\cond z)p(z)} -\\
    &-\mean_{\psi^{0..K}\sim q_\phi(\psi)}\mean_{z\sim q_\phi(z\cond\psi^0)}\log{\frac1K\sum_{k=1}^Kq_\phi(z\cond\psi^k)},\nonumber\\
\label{eq:lower_elbo}
    &\underline{\mathcal{L}}^q_K=\mean_{q_\phi(z)}\log{p(x\cond z)p(z)} -\\
    &-\mean_{\psi^{0..K}\sim q_\phi(\psi)}\mean_{z\sim q_\phi(z\cond\psi^0)}\log{\frac1{K+1}\sum_{k=0}^Kq_\phi(z\cond\psi^k)}.\nonumber
\end{align}
The surrogate objective $\underline{\mathcal{L}}^q_K$ is then used for optimization.

\subsection{SIVI Lower Bound}
Although it was shown that $\underline{\mathcal{L}}^q_0$ is a lower bound for ELBO, it has not been clear whether this holds for arbitrary $K$, and whether maximizing $\underline{\mathcal{L}}^q_K$ leads to a correct procedure.
Here, we show that $\underline{\mathcal{L}}^q_K$ is indeed a lower bound on ELBO $\mathcal{L}$.

\begin{theorem}
Consider $\mathcal{L}$ and $\underline{\mathcal{L}}^q_K$ defined as in Eq.~\eqref{eq:elbo} and \eqref{eq:lower_elbo}.
Then $\underline{\mathcal{L}}^q_K$ converges to $\mathcal{L}$ from below as $K \to \infty$, satisfying $\underline{\mathcal{L}}^q_K\leq\underline{\mathcal{L}}^q_{K+1}\leq\mathcal{L}$, and
\begin{align}
    &\underline{\mathcal{L}}^q_K=\mean_{\psi^{0..K}\sim q_\phi(\psi)}\mean_{q_\phi^K(z\,|\,\psi^{0..K})}\log{\frac{p(x\cond z)p(z)}{q_\phi^K(z\,|\,\psi^{0..K})}},\label{eq:sivil}\\
    &\text{where }q_\phi^K(z\,|\,\psi^{0..K})=\frac1{K+1}\sum_{k=0}^Kq_\phi(z\cond\psi^k).
\end{align}
\end{theorem}

The proof can be found in Appendix~\ref{app:lower-posterior}.

It can be seen from Eq.~\eqref{eq:sivil} that the surrogate objective $\underline{\mathcal{L}}^q_K$ proposed by \cite{yin2018} is actually the ELBO for a finite mixture approximation $q_\phi^K(z\cond\psi^0,\dots,\psi^K)$, that is averaged over all such mixtures (averaged over samples of $\psi^0,\dots,\psi^K\sim q_\phi(\psi)$).

\subsection{Semi-Implicit Priors}
\label{sec:sivi-prior}
Inspired by the derivation of the SIVI upper bound, we can derive the lower bound $\underline{\mathcal{L}}^{p}_K$ for the case of semi-implicit prior distributions.
Right now, for simplicity, assume an explicit approximate posterior $q_\phi(z)$, and a semi-implicit prior $p_\theta(z)=\int p_\theta(z|\zeta)p_\theta(\zeta)\,d\zeta$
\begin{gather}
\begin{multlined}
    \underline{\mathcal{L}}^p_K=\mean_{q_\phi(z)}\log{p(x\cond z)}-\\
    -\mean_{\zeta^{1..K}\sim p_\theta(\zeta)}\mean_{q_\phi(z)}\log\frac{q_\phi(z)}{\frac1K\sum_{k=1}^Kp_\theta(z\cond\zeta^k)},
\end{multlined}\\[1em]
\underline{\mathcal{L}}^{p}_K\leq \underline{\mathcal{L}}^{p}_{K+1}\leq\underline{\mathcal{L}}^{p}_\infty=\mathcal{L}.
\end{gather}
This bound has the same properties: it is non-decreasing in $K$ and is asymptotically exact.
To see why $\underline{\mathcal{L}}^{p}_K\leq\mathcal{L}$, one just needs to apply the Jensen's inequality for the logarithm:
\begin{equation}
    \begin{multlined}
    \mean_{\zeta^{1..K}\sim p_\theta(\zeta)}\mean_{q_\phi(z)}\log{\frac1K\sum_{k=1}^Kp_\theta(z\cond\zeta^k)}\leq\\[-0.5em]
    \begin{aligned}
        &\leq\mean_{q_\phi(z)}\log{\mean_{\zeta^{1..K}\sim p_\theta(\zeta)}\frac1K\sum_{k=1}^Kp_\theta(z\cond\zeta^k)}=\\
        &=\mean_{q_\phi(z)}\log{p_\theta(z)}.
    \end{aligned}        
    \end{multlined}
\end{equation}
To show that this bound is non-decreasing in $K$, one can refer to the proof of proposition~3 in the SIVI paper \cite[Appendix~A]{yin2018}.

Note that it is no longer possible to use the same trick to obtain the upper bound.
Still, we can obtain an upper bound using the variational representation of the KL-divergence \cite{nguyen2010}:
\begin{gather}
\begin{multlined}
    \KL(q_\phi(z)\,\|\,p_\theta(z))=\\
    =1+\sup_{g:\mathrm{dom}\,z\to\RR}\left\{\mean_{q_\phi(z)}g(z)-\mean_{p_\theta(z)}e^{g(z)}\right\}\geq\\
    \geq 1+\sup_{\eta}\left\{\mean_{q_\phi(z)}g(z, \eta)-\mean_{p_\theta(z)}e^{g(z, \eta)}\right\},
\end{multlined}\\[1em]
\begin{multlined}
    \overline{\mathcal{L}}^p_\eta=\mean_{q_\phi(z)}\log{p(x\cond z)}-\\
    -1-\mean_{q_\phi(z)}g(z, \eta)+\mean_{p_\theta(z)} e^{g(z, \eta)}.
\end{multlined}
\end{gather}
Here we substitute the maximization over all functions with a single parametric function.
In order to obtain a tighter bound, we can minimize this bound w.r.t. the parameters $\eta$ of function $g(z, \eta)$.

Note that in order to find the optimal value for $\eta$, one does not need to estimate the entropy term or the likelihood term of the objective:
\begin{equation}
    \eta^*=\arg\min_{\eta}\left[-\mean_{q_\phi(z)}g(z, \eta)+\mean_{p_\theta(z)}e^{g(z, \eta)}\right].
\end{equation}
This allows us to obtain a lower bound on the KL-divergence between two arbitrary (semi-)implicit distributions, and, consequently, results in an upper bound on the ELBO.

\subsection{Final Objective}
We can combine the bounds for the semi-implicit posterior and the semi-implicit prior to obtain the final lower bound
\begin{equation*}
    \begin{split}
    &\underline{\underline{\mathcal{L}}}_{K_1,K_2}^{q,p}=\mean_{q_\phi(z)}\log{p(x\cond z)}-\\
    &-\mean_{\psi^{0..K_1}\sim q_\phi(\psi)}\mean_{q_\phi(z\cond\psi^0)}\log\frac1{K_1+1}\sum_{k=0}^{K_1}q_\phi(z\cond\psi^k)+\\
    \end{split}
\end{equation*}
\begin{equation}
\label{eq:dsivilower}
    \begin{split}
    &+\mean_{\zeta^{1..K_2}\sim p_\theta(\zeta)}\mean_{q_\phi(z)}\log\frac1{K_2}\sum_{k=1}^{K_2}p_\theta(z\cond\zeta^k),
    \end{split}
\end{equation}
and the upper bound
\begin{equation}
\label{eq:dsiviupper}
\begin{split}
    \overline{{\mathcal{L}}}_{\eta}^{q,p}=&\mean_{q_\phi(z)}\log{p(x\cond z)}-\\
    &-1-\mean_{q_\phi(z)}g(z, \eta)+\mean_{p_\theta(z)}e^{g(z, \eta)}.
\end{split}
\end{equation}
The lower bound $\underline{\underline{\mathcal{L}}}_{K_1,K_2}^{q,p}$ is non-decreasing in both $K_1$ and $K_2$, and is asymptotically exact:
\begin{align}
    \underline{\underline{\mathcal{L}}}_{K_1,K_2}^{q,p}\leq \underline{\underline{\mathcal{L}}}_{K_1+1,K_2}^{q,p},\ \ &\ \ \ \underline{\underline{\mathcal{L}}}_{K_1,K_2}^{q,p}\leq \underline{\underline{\mathcal{L}}}_{K_1,K_2+1}^{q,p},\\
    \lim_{K_1,K_2\to\infty}&\underline{\underline{\mathcal{L}}}_{K_1,K_2}^{q,p}=\mathcal{L}.
\end{align}
We use the lower bound for optimization, whereas the upper bound may be used to estimate the gap between the lower bound and the true ELBO.
The final algorithm for DSIVI is presented in Algorithm~\ref{alg:dsivi}.
Unless stated otherwise, we use 1 MC sample to estimate the gradients of the lower bound (see Algorithm~\ref{alg:dsivi} for more details).
In the case where the prior distribution is explicit, one may resort to the upper bound $\overline{\mathcal{L}}_K^q$, proposed in SIVI~\cite{yin2018}.
\begin{algorithm}[t]
  \caption{\label{alg:dsivi} Doubly semi-implicit VI (and learning)}  
  \begin{algorithmic}  
    \Require{SI posterior $q_\phi(z)=\int q_\phi(z\cond\psi)q_\phi(\psi)\,d\psi$}
    \Require{SI prior $p_\theta(z)=\int p_\theta(z\cond\zeta)p_\theta(\zeta)\,d\zeta$}
    \Require{explicit log-likelihood $\log p(x\cond z)$}
    \State Variational inference (find $\phi$) and learning (find $\theta$)
    \For{$t \gets 1$ to $T$}
        \State $\psi^0,\dots,\psi^{K_1}\sim q_\phi(\psi)$ \Comment{Reparameterization}
        \State $\zeta^1,\dots,\zeta^{K_2}\sim p_\theta(\zeta)$ \Comment{Reparameterization}
        \State $z\sim q_\phi(z\cond\psi^0)$ \Comment{Reparameterization}
        \State Estimate $L_{LH}\simeq\log p(x\cond z)$
        \State $L_{E}\gets-\log\frac{1}{K_1+1}\sum_{k=0}^{K_1}q_\phi(z\cond\psi^k)$
        \State $L_{CE}\gets-\log\frac{1}{K_2}\sum_{k=1}^{K_2}p_\theta(z\cond\zeta^k)$
        \State $\hat{\underline{\underline{\mathcal{L}}}}_{K_1,K_2}^{q,p}\gets L_{LH}+L_{E}-L_{CE}$
        \State Use $\nabla_\phi\hat{\underline{\underline{\mathcal{L}}}}_{K_1,K_2}^{q,p}$ to update $\phi$
        \If{Variational learning}
            \State Use $\nabla_\theta\hat{\underline{\underline{\mathcal{L}}}}_{K_1,K_2}^{q,p}$ to update $\theta$
        \EndIf
    \EndFor
    \State Upper bound
    \For{$t \gets 1$ to $T$}
        \State $z\sim q_\phi(z)$ \Comment{Reparameterization}
        \State $z'\sim p_\theta(z)$ \Comment{Reparameterization}
        \State $L\gets -g(z, \eta)+e^{g(z',\eta)}$
        \State Use $-\nabla_\eta L$ to update $\eta$
    \EndFor
    \State Estimate ${\underline{\underline{\mathcal{L}}}}_{K_1,K_2}^{q,p}$ and ${\overline{\mathcal{L}}}_{\eta}^{q,p}$ using Eq.~\eqref{eq:dsivilower} and \eqref{eq:dsiviupper}\\
    \Return $\phi, \theta, \eta, {\underline{\underline{\mathcal{L}}}}_{K_1,K_2}^{q,p}, \overline{{\mathcal{L}}}_{\eta}^{q,p}$
  \end{algorithmic}  
\end{algorithm}

\section{APPLICATIONS}
\label{sec:app}
In this section we describe several settings that can benefit from semi-implicit prior distributions.
\subsection{VAE with Semi-Implicit Priors}
\label{sec:app-vae}

The default choice of the prior distribution $p(z)$ for the VAE model is the standard Gaussian distribution.
However, such choice is known to over-regularize the model \cite{tomczak2017,higgins2016}.

It can be shown that the so-called aggregated posterior distribution is the optimal prior distribution for a VAE in terms of the value of ELBO \cite{hoffman2016,tomczak2017}:
\begin{equation}
    p^*(z)=\frac{1}{N}\sum_{n=1}^N q_\phi(z\cond x_n),
\end{equation}
where the summation is over all training samples $x_n$, $n=1, \dotsc, N$. However, this extreme case leads to overfitting \cite{hoffman2016,tomczak2017}, and is highly computationally inefficient.
A possible middle ground is to consider the \emph{variational mixture of posteriors} prior distribution (the VampPrior) \cite{tomczak2017}:
\begin{equation}
    p^{Vamp}(z)=\frac{1}{K}\sum_{k=1}^K q_\phi(z\cond u_k).
\end{equation}
The VampPrior is defined as a mixture of $K$ variational posteriors $q_\phi(z\cond u_k)$ for a set of inducing points $\{u_k\}_{k=1}^K$.
These inducing points may be learnable (an ordinary VampPrior) or fixed at a random subset of the training data (VampPrior-data).
The VampPrior battles over-regularization by considering a flexible empirical prior distribution, being a mixture of fully-factorized Gaussians, and by coupling the parameters of the prior distribution and the variational posteriors.

There are two ways to improve this technique by using DSIVI.
We can regard the aggregated posterior $p^*(z)$ as a semi-implicit distribution:
\begin{equation}
\label{eq:vae-agg}
    p^*(z)=\frac{1}{N}\sum_{n=1}^N q_\phi(z|x_n)=\int q_\phi(z|x)p_{data}(x)dx.
\end{equation}
Next, we can use it as a semi-implicit prior and exploit the lower bound, presented in Section~\ref{sec:sivi-prior}:
\begin{equation}
\begin{multlined}
    \underline{\mathcal{L}}_{K}^p=\frac1N\sum_{n=1}^N\mean_{q_\phi(z\cond x_n)}\left[\log\frac{p(x_n\cond z)}{q_\phi(z\cond x_n)}+\right.\\
    \left.+\mean_{u_{1..K}\sim p_{data}(x)}\log\frac{1}{K}\sum_{k=1}^K q_\phi(z\cond u_k)\right].
\end{multlined}
\end{equation}
Note that the only difference from the training objective of VampPrior-data is that the inducing points $u_k$ are not \emph{fixed}, but are \emph{resampled} at each estimation of the lower bound.
As we show in the experiments, such reformulation of VampPrior-data drastically improves its test log-likelihood.

We can also consider an arbitrary semi-implicit prior distribution:
\begin{equation}
\label{eq:vae-sivi}
    p_\theta^{SI}(z)=\int p_\theta(z\cond\zeta)p_\theta(\zeta)\,d\zeta.
\end{equation}
For example, we consider a fully-factorized Gaussian conditional prior $p_\theta(z\cond\zeta)=\Normal(z\cond\zeta, \diag(\sigma^2))$ with mean $\zeta$ and trainable variances $\sigma_j^2$.
The implicit generator $p_\theta(\zeta)$ can be parameterized by an arbitrary neural network with weights $\theta$ that transforms a standard Gaussian noise $\eps$ to mixing parameters $\zeta$.
As we show in the experiments, such semi-implicit posterior outperforms VampPrior even though it does not couple the parameters of the prior and the variational posteriors.

We can also apply the importance-weighted lower bound \cite{burda2015} similarly to the importance weighted SIVAE~\cite{yin2018}, and obtain IW-DSIVAE, a lower bound on the IWAE objective for a variational autoencoder with a semi-implicit prior and a semi-implicit posterior.
The exact expression for this lower bound is presented in Appendix~\ref{app:iw-dsivae}.

\subsection{Variational Inference with Hierarchical Priors}
\label{sec:app-hp}
A lot of probabilistic models use hierarchical prior distributions: instead of a non-parametric prior $p(w)$ they use a parametric conditional prior $p(w\cond\alpha)$ with hyperparameters $\alpha$, and a hyperprior over these parameters $p(\alpha)$.
A discriminative model with such hierarchical prior may be defined as follows~\cite{neal1995,tipping2001,titsias2014,hernandez2015,shi2017,louizos2017compression}:
\begin{equation}
    p(t,w,\alpha\cond x)=p(t\cond x, w)p(w\cond\alpha)p(\alpha).
\end{equation}
A common way to perform inference in such models is to approximate the joint posterior $q_\phi(w,\alpha)\approx p(w,\alpha\cond X_{tr}, T_{tr})$ given the training data $(X_{tr}, T_{tr})$~\cite{hernandez2015,shi2017,louizos2017compression}.
Then the marginal approximate posterior $q_\phi(w)=\int q_\phi(w,\alpha)\,d\alpha$ is used to approximate the predictive distribution on unseen data $p(t\cond x, X_{tr}, T_{tr})$:
\begin{equation}
    \begin{multlined}
    p(t\cond x, X_{tr}, T_{tr})=\int p(t\cond x, w)p(w\cond X_{tr}, T_{tr})\,dw=\\
    \begin{aligned}
    &=\int p(t\cond x, w)\int p(w,\alpha\cond X_{tr}, T_{tr})\,d\alpha\,dw\approx\\
    &\approx \int p(t\cond x, w)\int q_\phi(w,\alpha)d\alpha\,dw=\\
    &=\int p(t\cond x, w)q_\phi(w)\,dw.
    \end{aligned}
    \end{multlined}
\end{equation}
The inference is performed by maximization of the following variational lower bound:
\begin{equation}
    \mathcal{L}^{joint}(\phi)=\mean_{q_\phi(w,\alpha)}\log\frac{p(t\cond x, w)p(w\cond\alpha)p(\alpha)}{q_\phi(w,\alpha)}.
\end{equation}
We actually are not interested in the joint posterior $q_\phi(w, \alpha)$, and we only need it to obtain the marginal posterior $q_\phi(w)$.
In this case we can reformulate the problem as variational inference with a semi-implicit prior $p(w)=\int p(w\cond\alpha)p(\alpha)\,d\alpha$ and a semi-implicit posterior $q_\phi(w)=\int q_\phi(w\cond\alpha)q_\phi(\alpha)\,d\alpha$:
\begin{equation}
    \mathcal{L}^{marginal}(\phi)=\mean_{q_\phi(w)}\log\frac{p(t\cond x, w)p(w)}{q_\phi(w)}.
\end{equation}
Then it can be shown that optimization of the second objective results in a better fit of the marginal posterior:
\begin{theorem}
Let $\phi_{j}$ and $\phi_{m}$ maximize $\mathcal{L}^{joint}$ and $\mathcal{L}^{marginal}$ correspondingly. Then
\begin{multline}
    \KL(q_{\phi_m}(w)\,\|\,p(w\cond X_{tr}, T_{tr}))\leq\\
    \KL(\,q_{\phi_{j}}(w)\,\,\|\,p(w\cond X_{tr}, T_{tr})).
\end{multline}
\end{theorem}
The proof can be found in Appendix~\ref{app:hierarchical}.

It means that if the likelihood function does not depend on the hyperparameters $\alpha$, it is beneficial to consider the semi-implicit formulation instead of the joint formulation of variational inference even if the approximation family stays exactly the same.
In the experiments, we show that the proposed DSIVI procedure matches the performance of direct optimization of $\mathcal{L}^{marginal}$, whereas joint VI performs much worse. 

\section{EXPERIMENTS}


\subsection{Variational Inference with Hierarchical Priors}
We consider a Bayesian neural network with a fully-factorized hierarchical prior distribution with a Gaussian conditional $p(w_{ij}\cond \alpha_{ij})=\Normal(w_{ij}\cond 0, \alpha_{ij}^{-1})$ and a Gamma hyperprior over the inverse variances $p(\alpha_{ij})=\mathrm{Gamma}(\alpha_{ij}\cond 0.5, 2)$.
Such hierarchical prior induces a fully-factorized Student's t-distribution with one degree of freedom as the marginal prior $p(w_{ij})=t(w_{ij}\cond\nu=1)$.
Note that in this case, we can estimate the marginal evidence lower bound directly.
We consider a fully-connected neural network with two hidden layers of 300 and 100 neurons on the MNIST dataset~\cite{lecun1998gradient}.
We train all methods with the same hyperparameters: we use batch size 200, use Adam optimizer~\cite{kingma2014adam} with default parameters, starting with learning rate $10^{-3}$, and train for 200 epochs, using linear learning rate decay.

We consider three different ways to perform inference in this model, the marginal inference, the joint inference, and DSIVI, as described in Section~\ref{sec:app-hp}.
For joint inference, we consider a fully-factorized joint approximate posterior $q_\phi(w,\alpha)=q_\phi(w)q_\phi(\alpha)$, with $q_\phi(w)$ being a fully-factorized Gaussian, and $q_\phi(\alpha)$ being a fully-factorized Log-Normal distribution. 
Such joint approximate posterior induces a fully-factorized Gaussian marginal posterior $q_\phi(w)$.
Therefore, we use a fully-factorized Gaussian posterior for the marginal inference and DSIVI.
Note that in this case, only the prior distribution is semi-implicit.
All models have been trained with the local reparameterization trick~\cite{kingma2015variational}.

We perform inference using these three different variational objectives, and then compare the true evidence lower bound $\mathcal{L}^{marginal}$ on the training set.
As the marginal variational approximation is the same in all four cases, the training ELBO can act as a proxy metric for the KL-divergence between the marginal approximate posterior and the true marginal posterior.
The results are presented in Figure~\ref{fig:students}.
DSIVI with as low as $K=10$ samples during training exactly matches the performance of the true marginal variational inference, whereas other approximations fall far behind.
All three methods achieve $97.7-98.0\%$ test set accuracy, and the test log-likelihood is approximately the same for all methods, ranging from $-830$ to $-855$.
However, the difference in the marginal ELBO is high.
The final values of the ELBO, its decomposition into train log-likelihood and the KL term, and the test log-likelihood are presented in Table~\ref{tab:students} in Appendix~\ref{app:hierarchical}.

\begin{figure}[t]
\includegraphics[width=\columnwidth]{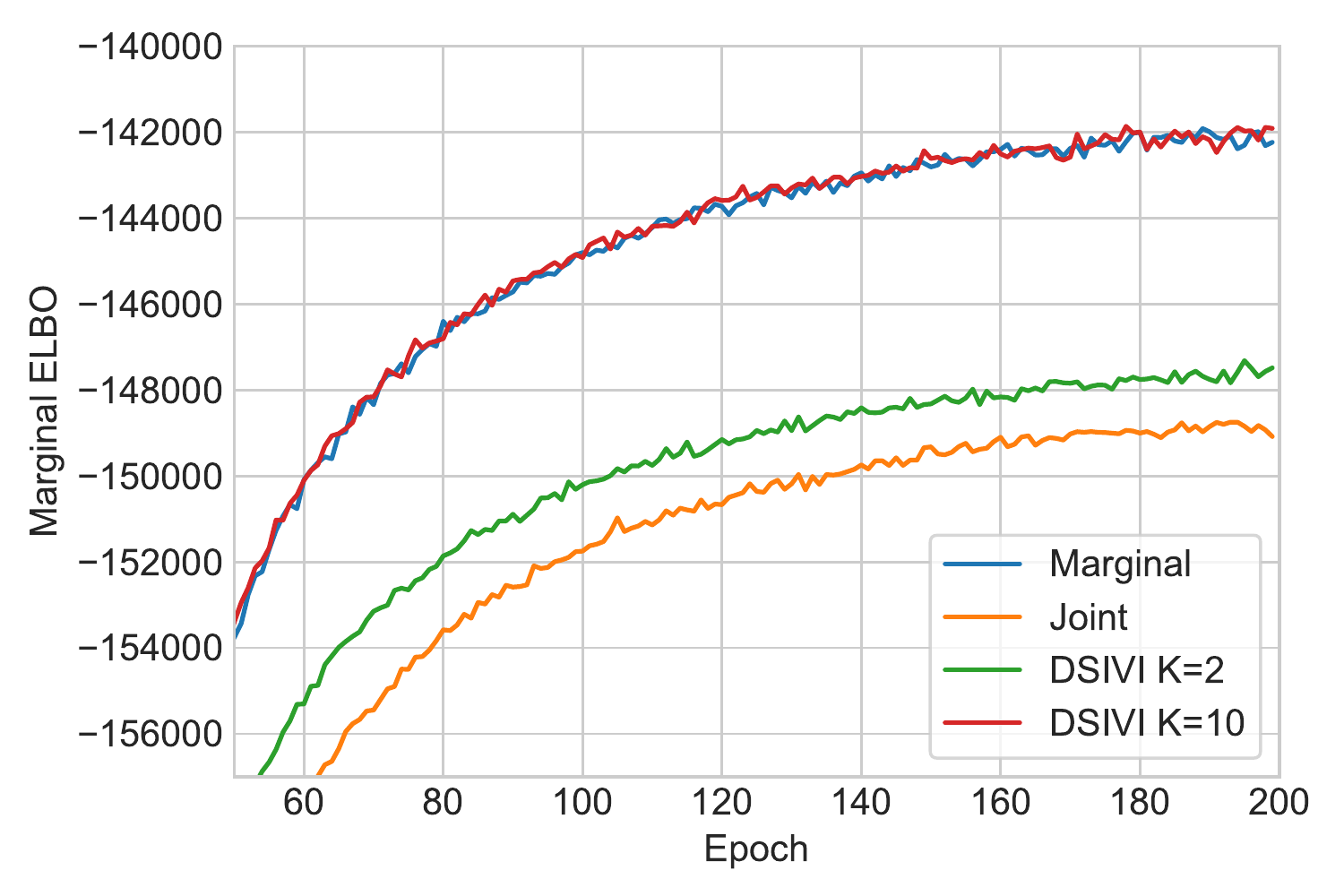}
\caption{Variational inference with a hierarchical prior. Models are trained using different variational objectives. The estimates of the marginal evidence lower bound are presented in this plot.}
\label{fig:students}
\end{figure}
\vspace{-0.5em}
\subsection{Comparison to Alternatives}
We compare DSIVI to other methods for implicit VI on a toy problem of approximating a centered standard Student's t-distribution $p(z)$ with 1 degrees of freedom with a Laplace distribution $q_\phi(z)$ by representing them as scale mixtures of Gaussians.
Namely, we represent $p(z)=\int \Normal(z\cond 0, \alpha^{-1}) \mathrm{Gamma}(\alpha\cond 0.5, 2)\,d\alpha$, and $q_\phi(z)=\int \Normal(z\cond \mu, \tau) \mathrm{Exp}(\tau\cond \lambda)\,d\tau$.
We train all methods by minimizing the corresponding approximations to the KL-divergence $\KL(q_\phi(z)\,\|\,p(z))$ w.r.t. the parameters $\mu$ and $\lambda$ of the approximation $q_\phi(z)$.

As baselines, we use prior methods for implicit VI: Adversarial Variational Bayes (AVB) \cite{mescheder2017}, which is a discriminator-based method, and Kernel Implicit Variational Inference (KIVI) \cite{shi2017}.
For AVB we fix architecture of the ``discriminator'' neural network to have 2 hidden layers with 3 and 4 hidden units with LeakyReLU ($\alpha = 0.2$) activation, and for KIVI we use fixed $\lambda = 0.001$ with varying number of samples.
For AVB we tried different numbers of training samples and optimization steps to optimize the discriminator at each step of optimizing over $\phi$.
We used Adam optimizer with learning rate $10^{-2}$ and one MC sample to estimate gradients w.r.t. $\phi$.

We report the KL-divergence $KL(q_\phi(z)\,\|\,p(z))$, estimated using 10000 MC samples averaged over 10 runs.
The results are presented in Figure~\ref{fig:comparison}.
DSIVI converges faster, is more stable, and only has one hyperparameter, the number of samples $K$ in the DSIVI objective.

\begin{figure}[t]
\vspace{-1em}
\hspace{-.15in}
\includegraphics[width=\columnwidth]{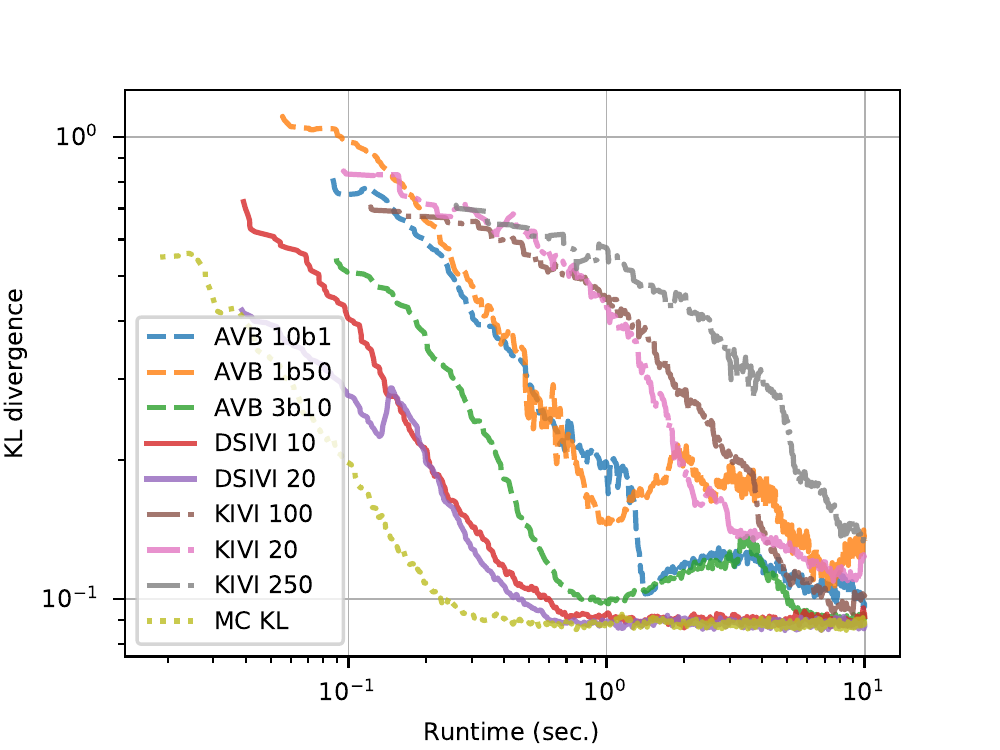}
\caption{Comparison of different techniques for \mbox{(semi-)implicit} VI. ``KIVI \emph{K}'' corresponds to KIVI, \emph{K} being the number of MC samples, used to approximate KL divergence; ``DSIVI \emph{K}'' corresponds to DSIVI with K=\emph{K}; ``AVB \emph{M}b\emph{K}'' corresponds to AVB with \emph{M} updates of discriminator per one update of $\phi$ and \emph{K} MC samples to estimate the discriminator's gradients. ``MC KL'' corresponds to direct stochastic minimization of the KL divergence.}
\label{fig:comparison}
\end{figure}

\subsection{Sequential Approximation}

We illustrate the expressiveness of DSIVI with implicit prior and posterior distributions on the following toy problem.
Consider an explicit distribution~$p(z)$.
We would like to learn a semi-implicit distribution $q_{\phi_1}(z)=\mean_{q_{\phi_1}(\psi)}[q_{\phi_1}(z\cond\psi)]$ to match $p(z)$.
During the first step, we apply DSIVI to tune the parameters~$\phi_1$ so as to minimize $\KL(q_{\phi_1}(z)\,\|\,p(z))$.
Then, we take the trained semi-implicit $q_{\phi_1}(z)$ as a new target for $z$ and tune $\phi_2$ minimizing $\KL(q_{\phi_2}(z)\,\|\,q_{\phi_1}(z))$.
After we repeat the iterative process $k$~times, $q_{\phi_k}(z)$ obtained through minimization of $\KL(q_{\phi_k}(z)\,\|\,q_{\phi_{k-1}}(z))$ should still match~$p(z)$.

In our experiments, we follow \cite{yin2018} and model $q_{\phi_i}(\psi)$ by a multi-layer perceptron (MLP) with layer widths [30,60,30] with ReLU activations and a ten-dimensional standard normal noise as its input.
We also fix all conditionals $q_{\phi_i}(z \cond \psi) = \Normal(z \cond \psi, \sigma^2I)$, $\sigma^2 = 0.1$.
We choose $p(z)$ to be either a one-dimensional mixture of Gaussians or a two-dimensional ``banana'' distribution.
In Figure~\ref{fig:kl_degradation} we plot values of $\KL(q_{\phi_i}(z)\,\|\,p(z))$, $i=1, \dotsc, 9$ for different values of $K_1 = K_2 = K$ (see Algorithm~\ref{alg:dsivi}) when $p(z)$ is a one-dimensional mixture of Gaussians (see Appendix~\ref{app:seq-approx} for a detailed description and additional plots).
In Figure~\ref{fig:mixture_end} we plot the approximate PDF of $q_{\phi_k}(z)$ after 9 steps for different values of $K$.
As we can see, even though both ``prior'' and ``posterior'' distributions are semi-implicit, the algorithm can still accurately learn the original target distribution after several iterations.

\begin{figure}[t]
\hspace{-.15in}
\includegraphics[width=\columnwidth]{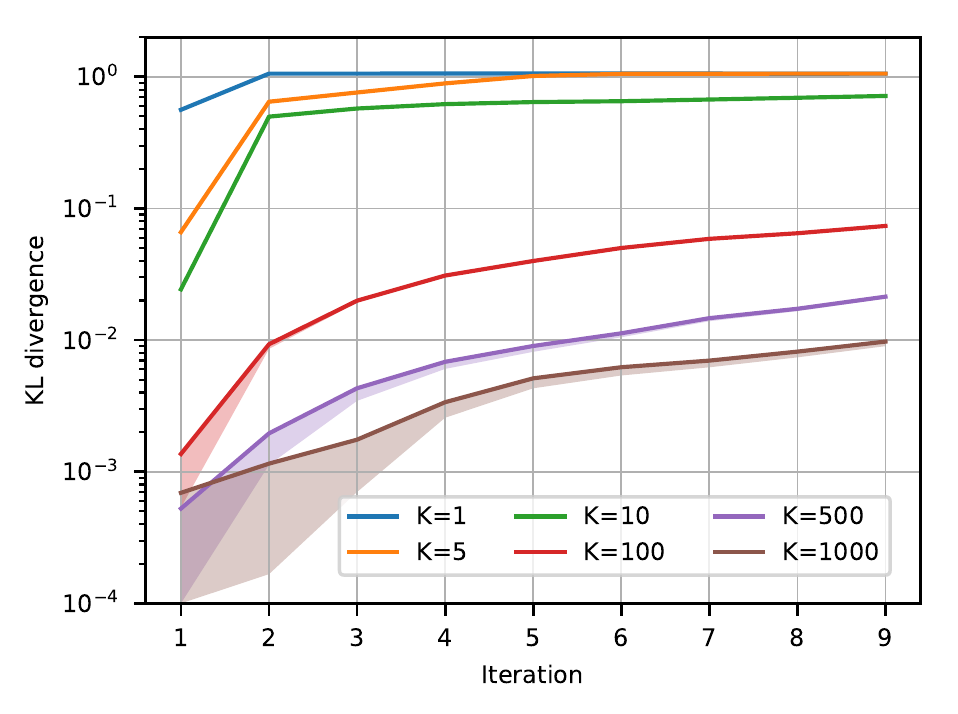}
\caption{Sequential approximation. Area is shaded between lower and upper bounds of $\KL(q_{\phi_i}(z)\,\|\,p(z))$ for different \emph{training} values of $K_1 = K_2 = K$, and the solid lines represent the corresponding upper bounds. During \emph{evaluation}, $K = 10^4$ is used. Here $p(z)$ is a one dimensional Gaussian mixture (see Appendix~\ref{app:seq-approx} for details.)  Lower is better.}
\label{fig:kl_degradation}
\end{figure}

\begin{figure}[t]
\hspace{-.15in}
\includegraphics[width=\columnwidth]{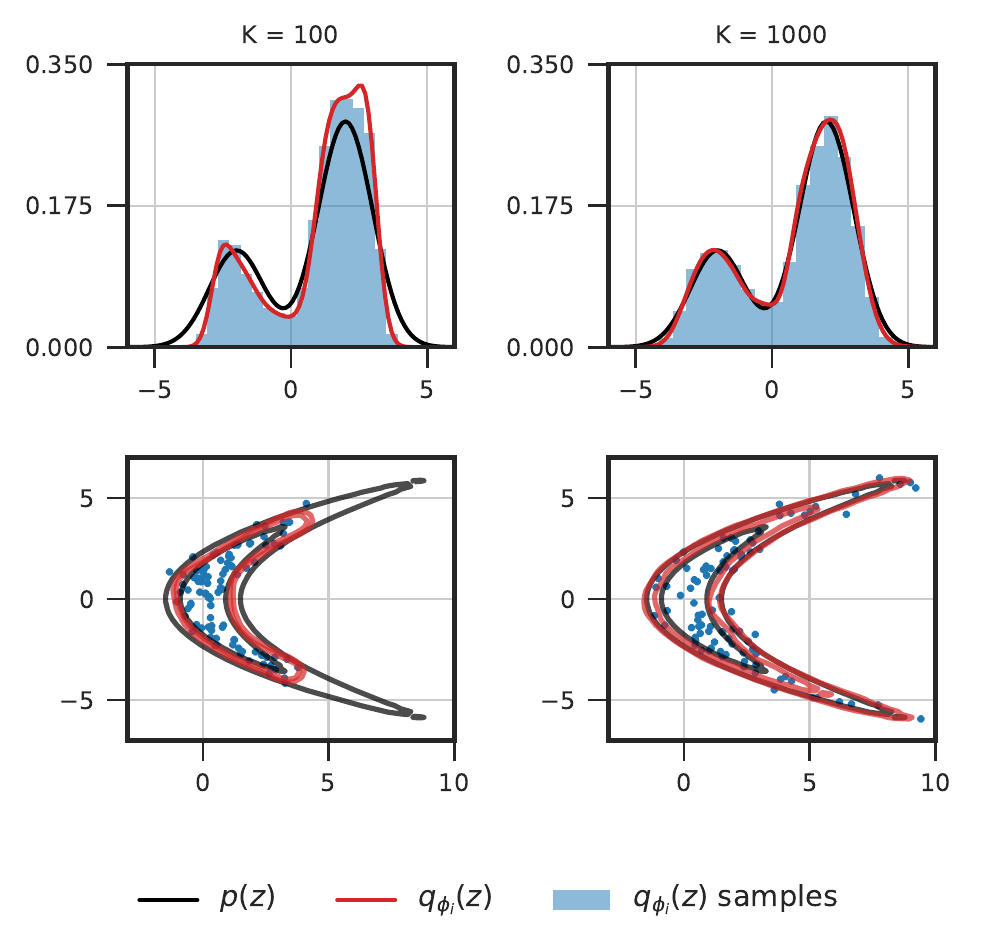}
\caption{Learned probability distributions $q_{\phi_k}(z)$ after 9 iterations of sequential approximation for $K_1 = K_2 = 100$ and $1000$ (red), and the two original priors $p(z)$ (black). During evaluation, $K = 10^4$.}
\label{fig:mixture_end}
\end{figure}

\subsection{VAE with Semi-Implicit Optimal Prior}
We follow the same experimental setup and use the same hyperparameters, as suggested for VampPrior \cite{tomczak2017}.
We consider two architectures, the VAE and the HVAE (hierarchical VAE, \cite{tomczak2017}), applied to the MNIST dataset with dynamic binarization~\cite{salakhutdinov2008quantitative}.
In both cases, all distributions (except the prior) have been modeled by fully-factorized neural networks with two hidden layers of 300 hidden units each.
We used 40-dimensional latent vectors $z$ (40-dimensional $z_1$ and $z_2$ for HVAE) and  Bernoulli likelihood with dynamic binarization for the MNIST dataset.
As suggested in the VampPrior paper, we used 500 pseudo-inputs for VampPrior-based models in all cases (higher number of pseudo-inputs led to overfitting).
To measure the performance of all models, we bound the test log-likelihood with the IWAE objective \cite{burda2015} with 5000 samples for the VampPrior-based methods, and estimate the corresponding IW-DSIVAE lower bound with $K=20000$ for the DSIVI-based methods (see Appendix~\ref{app:iw-dsivae} for more details).

\begin{table}[t]
\caption{We compare VampPrior with its semi-implicit modifications, DSIVI-agg and DSIVI-prior. We report the the IWAE objective $\mathcal{L}^S$ for VampPrior-data, and the corresponding lower bound $\underline{\mathcal{L}}_K^{p,S}$ for DSIVI-based methods (see Appendix~\ref{app:iw-dsivae}). Only the prior distribution is semi-implicit.} \label{tab:vae}
\begin{center}
\begin{tabular}{lr}
Method & LL\\
\hline
VAE+VampPrior-data & $-85.05$\\
VAE+VampPrior & $-82.38$\\
VAE+DSIVI-prior (K=2000) & $\geq-82.27$\\
VAE+DSIVI-agg (K=500) & $\geq -83.02$\\
VAE+DSIVI-agg (K=5000) & $\geq \mathbf{-82.16}$\\
\hline
HVAE+VampPrior-data & $-81.71$\\
HVAE+VampPrior & $-81.24$\\
HVAE+DSIVI-agg (K=5000) & $\geq \mathbf{-81.09}$\\
\hline
\end{tabular}
\end{center}
\vspace{-1em}
\end{table}

We consider two formulations, described in Section~\ref{sec:app-vae}: DSIVI-agg stands for the semi-implicit formulation of the aggregated posterior~\eqref{eq:vae-agg}, and DSIVI-prior stands for a general semi-implicit prior~\eqref{eq:vae-sivi}.
For the DSIVI-prior we have used a fully-factorized Gaussian conditional $p(z\cond\zeta)=\Normal(z\cond\zeta,\diag(\sigma^2))$, where the mixing parameters $\zeta$ are the output of a fully-connected neural network with two hidden layers with 300 and 600 hidden units respectively, applied to a 300-dimensional standard Gaussian noise $\epsilon$.
The first and second hidden layers were followed by ReLU non-linearities, and no non-linearities were applied to obtain $\zeta$.
We did not use warm-up \cite{tomczak2017} with DSIVI-prior.

The results are presented in Table~\ref{tab:vae}.
DSIVI-agg is a simple modification of VampPrior-data that significantly improves the test log-likelihood, and even outperforms the VampPrior with trained inducing inputs.
DSIVI-prior outperforms VampPrior even without warm-up and without coupling the parameters of the prior and the variational posteriors.

\section{CONCLUSION}
We have presented DSIVI, a general-purpose framework that allows to perform variational inference and variational learning when both the approximate posterior distribution and the prior distribution are semi-implicit.
DSIVI provides an asymptotically exact lower bound on the ELBO, and also an upper bound that can be made arbitrarily tight.
It allows us to estimate the ELBO in any model with semi-implicit distributions, which was not the case for other methods.
We have shown the effectiveness of DSIVI applied to a range of problems, e.g. models with hierarchical priors and variational autoencoders with semi-implicit empirical priors.
In particular, we show how DSIVI-based treatment improves the performance of VampPrior, the current state-of-the-art prior distribution for VAE. 

\section*{ACKNOWLEDGMENTS}
We would like to thank Arsenii Ashukha and Kirill Neklyudov for valuable discussions. Dmitry Molchanov and Valery Kharitonov were supported by Samsung Research, Samsung Electronics.


\bibliographystyle{abbrv}
\bibliography{main.bib}

\clearpage
\appendix
\section{Proof of the SIVI Lower Bound for Semi-Implicit Posteriors}
\label{app:lower-posterior}
\setcounter{theorem}{0}
\begin{theorem}
Consider $\mathcal{L}$ and $\underline{\mathcal{L}}_K^q$ defined as in Eq.~\eqref{eq:elbo} and \eqref{eq:lower_elbo}.
Then $\underline{\mathcal{L}}_K^q$ converges to $\mathcal{L}$ from below as~$K \to~\infty$, satisfying $\underline{\mathcal{L}}_K^q\leq\underline{\mathcal{L}}_{K+1}^q\leq\mathcal{L}$, and 
\begin{align}
    &\underline{\mathcal{L}}_K^q=\mean_{\psi^{0..K}\sim q_\phi(\psi)}\mean_{q_\phi^K(z\,|\,\psi^{0..K})}\log{\frac{p(x\cond z)p(z)}{q_\phi^K(z\,|\,\psi^{0..K})}},\label{eq:sivil-2}\\
    &\text{where }q_\phi^K(z\,|\,\psi^{0..K})=\frac1{K+1}\sum_{k=0}^Kq_\phi(z\cond\psi^k).
\end{align}
\end{theorem}
\begin{proof}
For brevity, we denote $\mean_{\psi^{0..K}\sim q_\phi(\psi)}$ as $\mean_{\psi^{0..K}}$ and $\mean_{z\sim q_\phi^K(z\cond\psi^{0..K})}$ as $\mean_{z\cond\psi^{0..K}}$.
First, notice that due to the symmetry in the indices, the regularized lower bound $\underline{\mathcal{L}}_K^q$ does not depend on the index in the conditional $q_\phi(z\cond\psi^i)$:
\begin{align}
    \underline{\mathcal{L}}_K^q&=\mean_{\psi^{0..K}}\mean_{z\cond\psi^0}\log\frac{p(x,z)}{q_\phi^K(z\cond\psi^{0..K})}=\\
    &=\mean_{\psi^{0..K}}\mean_{z\cond\psi^i}\log\frac{p(x,z)}{q_\phi^K(z\cond\psi^{0..K})}.
\end{align}
Therefore, we can rewrite $\underline{\mathcal{L}}_K^q$ as follows:
\begin{align}
    &\underline{\mathcal{L}}_K^q=\frac1{K+1}\sum_{i=0}^K\underline{\mathcal{L}}_K^q=\\
    &=\frac1{K+1}\sum_{i=0}^K\mean_{\psi^{0..K}}\mean_{z\cond\psi^i}\log\frac{p(x,z)}{q_\phi^K(z\cond\psi^{0..K})}=\\
    &=\mean_{\psi^{0..K}}\mean_{z\cond\psi^{0..K}}\log\frac{p(x,z)}{q_\phi^K(z\cond\psi^{0..K})}.
\end{align}
Note that it is just the value of the evidence lower bound with the approximate posterior $q_\phi^K(z\cond\psi^{0..K})$, averaged over all values of $\psi^{0..K}$.
We can also use that $\mean_{\psi^{0..K}}q_\phi^K(z\cond\psi^{0..K})=q_\phi(z)$ to rewrite the true ELBO in the same expectations:
\begin{align}
    &\mathcal{L}=\mean_{q_\phi(z)}\log\frac{p(x,z)}{q_\phi(z)}=\\
    &=\mean_{\psi^{0..K}}\mean_{z\cond\psi^{0..K}}\log\frac{p(x,z)}{q_\phi(z)}.
\end{align}
We want to prove that $\mathcal{L}\geq\underline{\mathcal{L}}_K^q$.
Consider their difference $\mathcal{L}-\underline{\mathcal{L}}_K^q$:
\begin{align}
    &\mathcal{L}-\underline{\mathcal{L}}_K^q=\\
    &=\mean_{\psi^{0..K}}\mean_{z\cond\psi^{0..K}}\log\frac{q_\phi^K(z\cond\psi^{0..K})}{q_\phi(z)}=\\
    &=\mean_{\psi^{0..K}}\KL\left(q_\phi^K(z\cond\psi^{0..K})\,\|\,q_\phi(z)\right)\geq 0.
\end{align}
We can use the same trick to prove that this bound is non-decreasing in $K$.
First, let's use the symmetry in the indices once again, and rewrite $\underline{\mathcal{L}}_K^q$ and $\underline{\mathcal{L}}_{K+1}^q$ in the same expectations:
\begin{align}
    &\underline{\mathcal{L}}_K^q=\mean_{\psi^{0..K}}\mean_{z\cond\psi^{0..K}}\log\frac{p(x,z)}{q_\phi^K(z\cond\psi^{0..K})}=\\
    &=\mean_{\psi^{0..K+1}}\mean_{z\cond\psi^{0..K}}\log\frac{p(x,z)}{q_\phi^K(z\cond\psi^{0..K})},\\
    &\underline{\mathcal{L}}_{K+1}^q=\mean_{\psi^{0..K+1}}\mean_{z\cond\psi^0}\log\frac{p(x,z)}{q_\phi^{K+1}(z\cond\psi^{0..K+1})}=\\
    &=\mean_{\psi^{0..K+1}}\mean_{z\cond\psi^{0..K}}\log\frac{p(x,z)}{q_\phi^{K+1}(z\cond\psi^{0..K+1})}.
\end{align}
Then their difference would be equal to the expected KL-divergence, hence being non-negative:
\begin{align}
    &\underline{\mathcal{L}}_{K+1}^q-\underline{\mathcal{L}}_{K}^q=\\
    &=\mean_{\psi^{0..K+1}}\mean_{z\cond\psi^{0..K}}\log\frac{q_\phi^K(z\cond\psi^{0..K})}{q_\phi^{K+1}(z\cond\psi^{0..K+1})}=\\
    &=\mean_{\psi^{0..K+1}}\KL\left(q_\phi^K(z\cond\psi^{0..K})\,\|\,q_\phi^{K+1}(z\cond\psi^{0..K+1}))\right) \nonumber\\
    &\geq 0.\nonumber
\end{align}
\end{proof}


\section{Importance Weighted Doubly Semi-Implicit VAE}
\label{app:iw-dsivae}
The standard importance-weighted lower bound for VAE is defined as follows:
\begin{equation}
    \log p(x)\geq \mathcal{L}^S=\mean_{z^{1..S}\sim q_\phi(z)}\log\frac1S\sum_{i=1}^S\frac{p(x\cond z^i)p(z^i)}{q_\phi(z_i\cond x)}
\end{equation}
We propose IW-DSIVAE, a new lower bound on the IWAE objective, that is suitable for VAEs with semi-implicit priors and posteriors:
\begin{gather}
    \begin{multlined}
        \underline{\underline{\mathcal{L}}}_{K_1,K_2}^{q,p,S}=\mean_{\psi^{1..K_1}\sim q_\phi(\psi)}\mean_{\zeta^{1..K_2}\sim p_\theta(\zeta)}\left[\phantom{\frac1S}\right.\\
        \mean_{(z^1,\hat{\psi}^1),\dots,(z^S,\hat{\psi}^S)\sim q_\phi(z,\psi)}\left[\phantom{\frac1S}\right.\\
        \left.\left.\log\frac1S\sum_{i=1}^S\frac{p(x\cond z^i)\frac{1}{K_2}\sum_{k=1}^{K_2}p_\theta(z^i\cond\zeta^k)}{\frac{1}{K_1+1}(q_\phi(z^i\cond\hat{\psi}^i)+\sum_{k=1}^{K_1}q_\phi(z^i\cond\psi^k))}\right]\right].
    \end{multlined}
\end{gather}
This objective is a lower bound on the IWAE objective ($\underline{\underline{\mathcal{L}}}_{K_1,K_2}^{q,p,S}\leq \mathcal{L}^S$), is non-decreasing in both $K_1$ and $K_2$, and is asymptotically exact ($\underline{\underline{\mathcal{L}}}_{\infty,\infty}^{q,p,S}=\mathcal{L}^S$).

\section{Variational inference with hierarchical priors}
\label{app:hierarchical}
\begin{theorem}
Consider two different variational objectives $\mathcal{L}^{joint}$ and $\mathcal{L}^{marginal}$. Then
\begin{align}
    &\mathcal{L}^{joint}(\phi)=\mean_{q_\phi(w,\alpha)}\log\frac{p(t\cond x, w)p(w\cond\alpha)p(\alpha)}{q_\phi(w,\alpha)}\\
    &\mathcal{L}^{marginal}(\phi)=\mean_{q_\phi(w)}\log\frac{p(t\cond x, w)p(w)}{q_\phi(w)}
\end{align}
Let $\phi_{j}$ and $\phi_{m}$ maximize $\mathcal{L}^{joint}$ and $\mathcal{L}^{marginal}$ correspondingly.
Then $q_{\phi_m}(w)$ is a better fit for the marginal posterior that $q_{\phi_j}(w)$ in terms of the KL-divergence:
\begin{multline}
    \label{eq:kl-inequality2}
    \KL(q_{\phi_m}(w)\,\|\,p(w\cond X_{tr}, T_{tr}))\leq\\
    \KL(\,q_{\phi_{j}}(w)\,\,\|\,p(w\cond X_{tr}, T_{tr}))
\end{multline}
\end{theorem}
\begin{proof}
Note that maximizing $\mathcal{L}^{marginal}(\phi)$ directly minimizes $\KL(q_{\phi}(w)\,\|\,p(w\cond X_{tr}, T_{tr}))$, as $\mathcal{L}^{marginal}(\phi)+\KL(q_{\phi}(w)\,\|\,p(w\cond X_{tr}, T_{tr}))=const$.
The sought-for inequality~\eqref{eq:kl-inequality2} then immediately follows from $\mathcal{L}^{marginal}(\phi_m)\geq\mathcal{L}^{marginal}(\phi_j)$.
\end{proof}

To see the cause of this inequality more clearly, consider $\mathcal{L}^{joint}(\phi)$:
\begin{align}
    &\mathcal{L}^{joint}(\phi)=\mean_{q_\phi(w,\alpha)}\log\frac{p(t\cond x, w)p(w\cond\alpha)p(\alpha)}{q_\phi(w,\alpha)}=\\
    &=\mean_{q_\phi(w)}\log p(t\cond x, w)-\KL(q_\phi(w,\alpha)\,\|\,p(w,\alpha))=\\
    &=\mean_{q_\phi(w)}\log p(t\cond x, w)-\KL(q_\phi(w)\,\|\,p(w))-\\
    &-\mean_{q_\phi(w)}\KL(q_\phi(\alpha\cond w)\,\|\,p(\alpha\cond w))=\\
    &=\mathcal{L}^{marginal}(\phi)-\mean_{q_\phi(w)}\KL(q_\phi(\alpha\cond w)\,\|\,p(\alpha\cond w))
\end{align}
If $\mathcal{L}^{joint}$ and $\mathcal{L}^{marginal}$ coincide, the inequality~\eqref{eq:kl-inequality2} becomes an equality.
However, $\mathcal{L}^{joint}$ and $\mathcal{L}^{marginal}$ only coincide if the reverse posterior $q_\phi(\alpha\cond w)$ is an exact match for the reverse prior $p(\alpha\cond w)$.
Due to the limitations of the approximation family of the joint posterior, this is not the case in many practical applications.
In many cases \cite{hernandez2015,louizos2017compression} the joint approximate posterior is modeled as a factorized distribution $q_\phi(w,\alpha)=q_\phi(w)q_\phi(\alpha)$.
Therefore in the case of the joint variational inference, we optimize a lower bound on the marginal ELBO and therefore obtain a sub-optimal approximation.

\begin{table}[h]
\caption{The values of the marginal ELBO, the train negative log-likelihood, the KL-divergence between the marginal posterior $q_\phi(w)$ and the marginal prior $p_\phi(w)$, and the test-set accuracy and negative log-likelihood for different inference procedures for a model with a standard Student's prior. The predictive distribution during test-time was estimated using 200 samples from the marginal posterior $q_\phi(w)$} \label{tab:students}
\begin{center}
\resizebox{\linewidth}{!}{
\begin{tabular}{l|ccc|cc}
       & \multicolumn{3}{c}{\textbf{Train}} & \multicolumn{2}{|c}{\textbf{Test}}\\
\textbf{Method} & \textbf{ELBO} & \textbf{NLL} & \textbf{KL} & \textbf{Acc.} & \textbf{NLL}\\
\hline
Marginal & $\mathbf{-1.42\times10^5}$ & $7.2\times10^3$ & $1.35\times10^5$ & $97.80$ & $855$\\
Joint & $-1.48\times10^5$ & $6.7\times10^3$ & $1.42\times10^5$ & $97.74$ & $831$\\
DSIVI(K=2) & $-1.47\times10^5$ & $7.0\times10^3$ & $1.41\times10^5$ & $97.75$ & $846$\\
DSIVI(K=10) & $\mathbf{-1.42\times10^5}$ & $7.2\times10^3$ & $1.35\times10^5$ & $97.76$ & $843$
\end{tabular}}
\end{center}
\end{table}

\section{Toy data for sequential approximation}
\label{app:seq-approx}
For sequential approximation toy task, we follow \cite{yin2018} and use the following target distributions.
For one-dimensional Gaussian mixture, $p(z) = 0.3 \Normal(z \,|\, {-2}, 1) + 0.7 \Normal(z \,|\, 2, 1)$.
For the ``banana'' distribution, $p(z_1, z_2) = \Normal(z_1 \,|\, z_2^2 / 4, 1) \Normal(z_2 \,|\, 0, 4)$.

For both target distributions, we optimize the objective using Adam optimizer with initial learning rate $10^{-2}$ and decaying it by $0.5$ every 500 steps.
On each iteration of sequential approximation, we train for 5000 steps.
We reinitialize all trainable parameters and optimizer statistics before each iteration.
Before each update of the parameters, we average 200 Monte Carlo samples of the gradients.
During evaluation, we used $10^5$ Monte Carlo samples to estimate the expectations involved in the lower and upper bounds on KL divergence.

\begin{figure}[h]
\hspace{-.15in}
\includegraphics[width=\columnwidth]{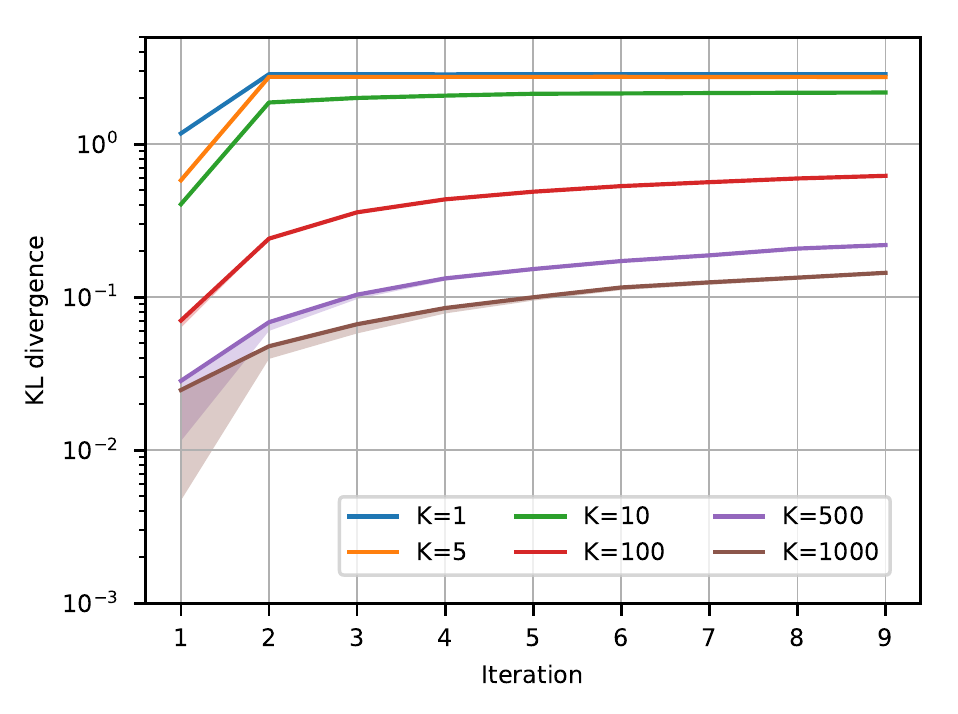}
\caption{Sequential approximation. Area is shaded between lower and upper bounds of $\KL(q_{\phi_i}(z)\,\|\,p(z))$ for different \emph{training} values of $K_1 = K_2 = K$, and the solid lines represent the corresponding upper bounds. During \emph{evaluation}, $K = 10^4$ is used. Here $p(z)$ is a two-dimensional ``banana``. Lower is better.}
\label{fig:kl_degradation_banana}
\end{figure}

\onecolumn

\begin{figure}[t]
\hspace{-.15in}
\centering{\includegraphics[width=\textwidth]{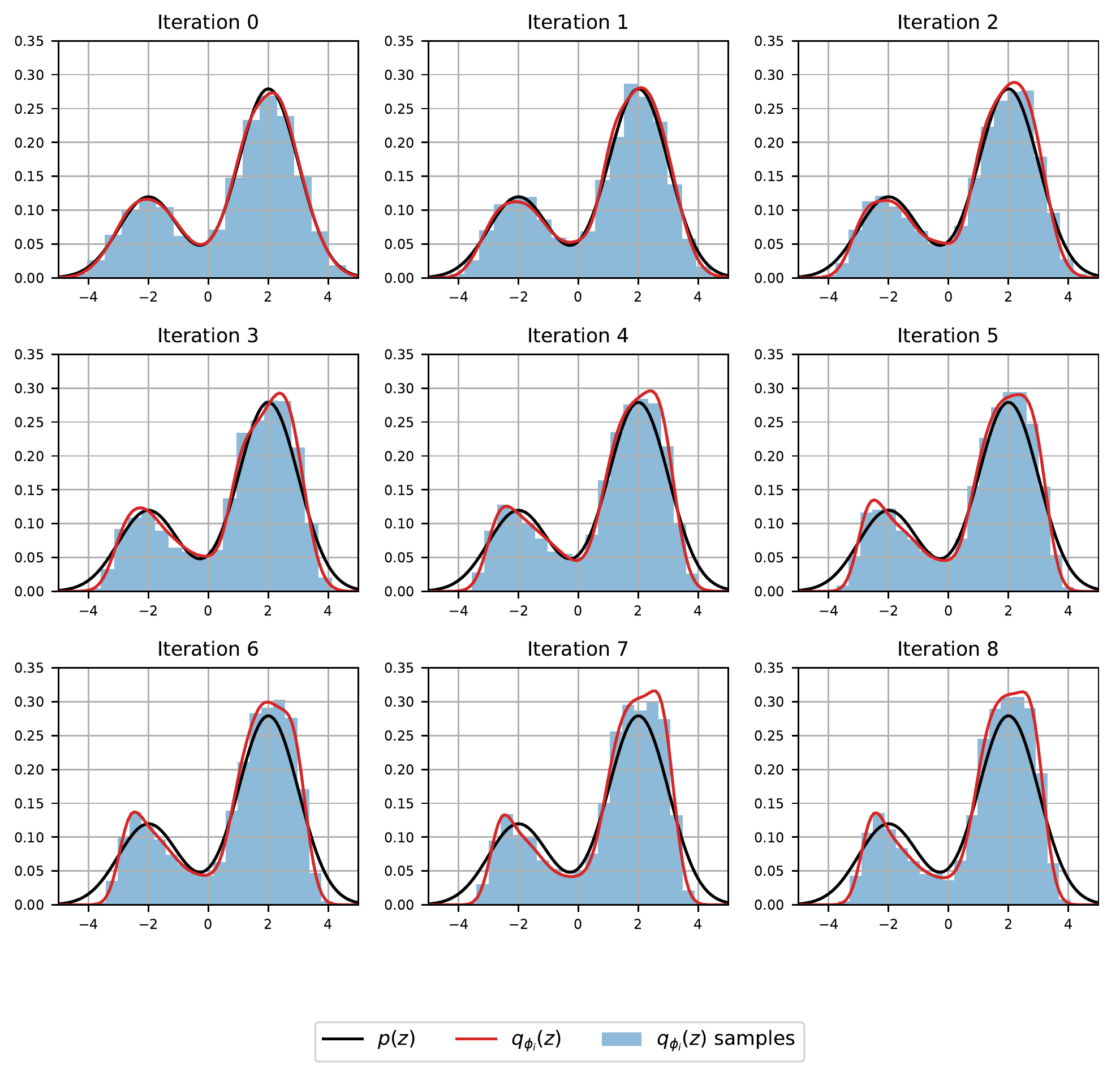}}
\caption{Learned distributions after each iteration for Gaussian mixture target distribution, $K = 100$ during training.}
\label{fig:mixture_100}
\end{figure}

\begin{figure}[t]
\hspace{-.15in}
\centering{\includegraphics[width=\textwidth]{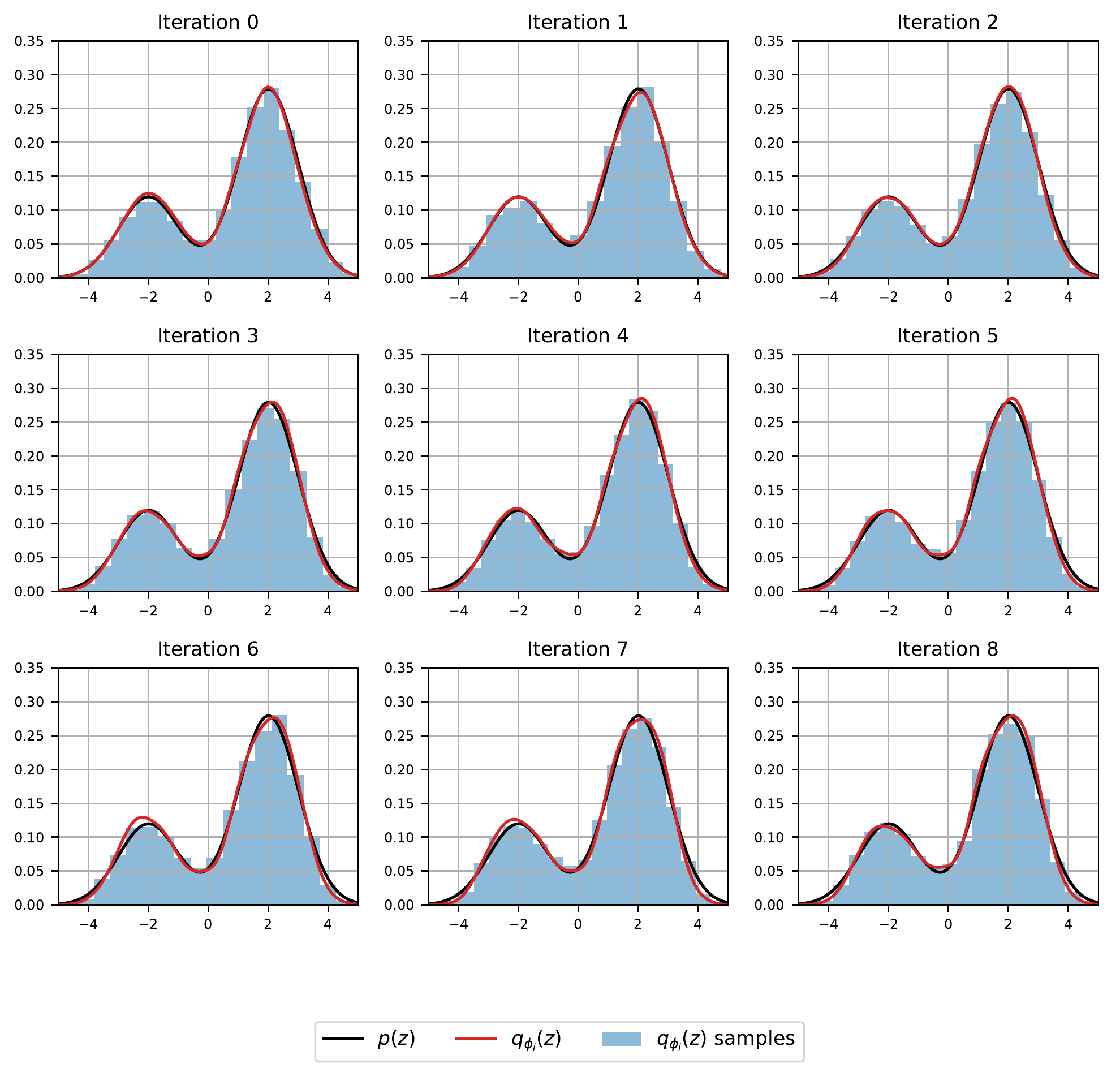}}
\caption{Learned distributions after each iteration for Gaussian mixture target distribution, $K = 1000$ during training.}
\label{fig:mixture_1000}
\end{figure}

\begin{figure}[t]
\hspace{-.15in}
\centering{\includegraphics[width=\textwidth]{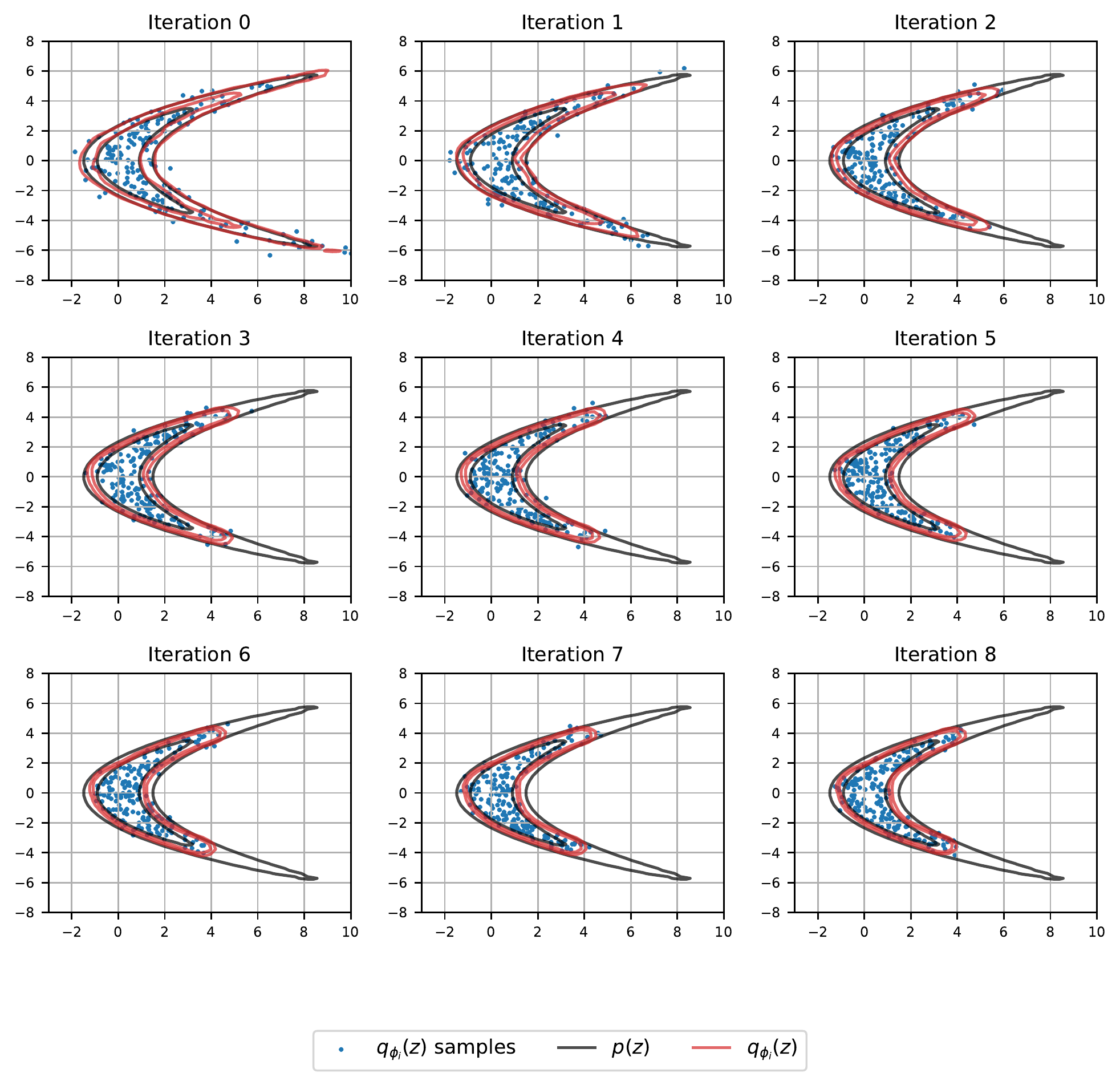}}
\caption{Learned distributions after each iteration for ``banana'' target distribution, $K = 100$ during training.}
\label{fig:banana_100}
\end{figure}

\begin{figure}[t]
\hspace{-.15in}
\centering{\includegraphics[width=\textwidth]{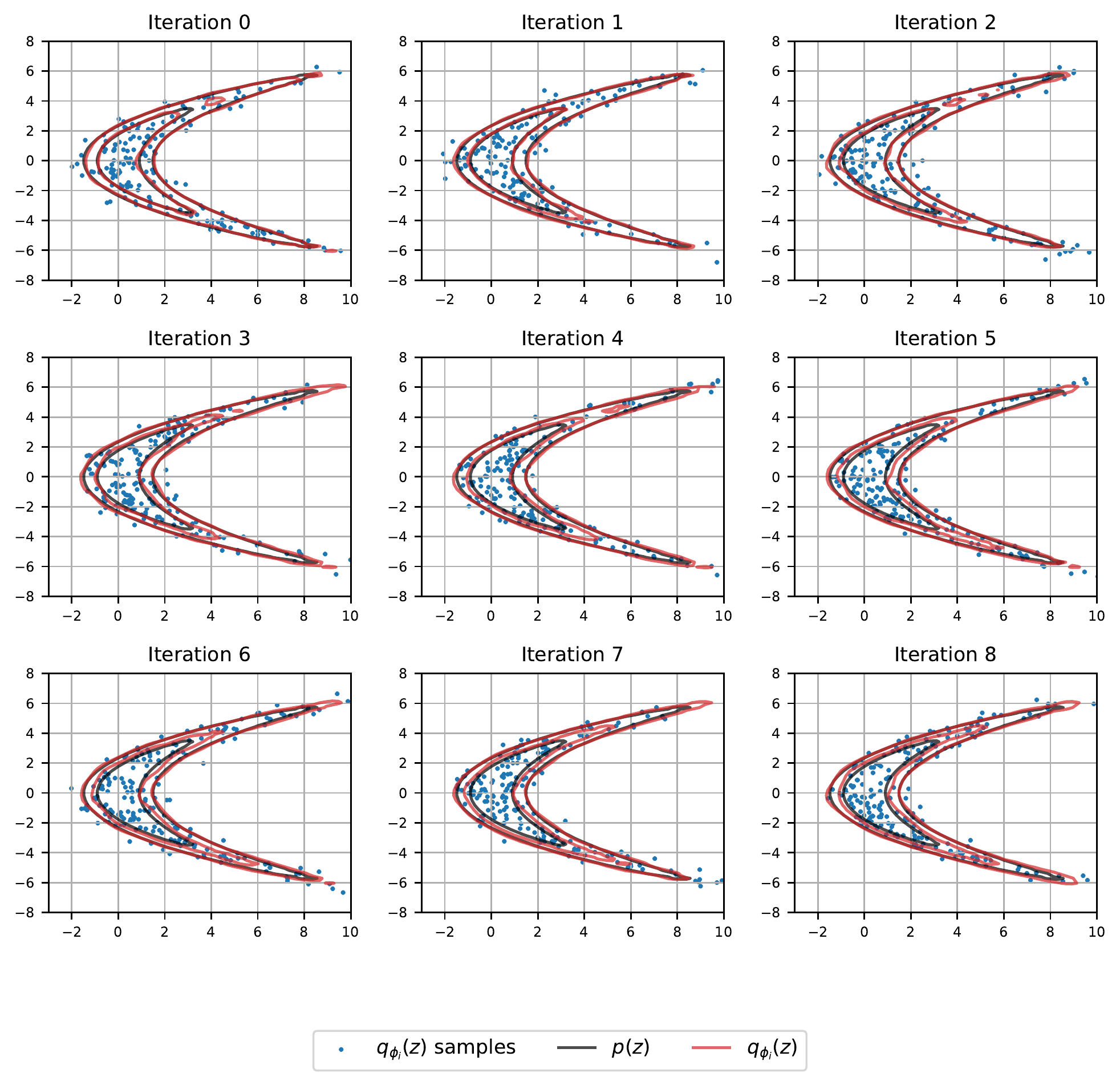}}
\caption{Learned distributions after each iteration for ``banana'' target distribution, $K = 1000$ during training.}
\label{fig:banana_1000}
\end{figure}

\end{document}